\documentclass[11pt]{article}
\usepackage[margin=1in]{geometry}
\usepackage{times}
\usepackage{amsmath, amssymb, amsthm}
\usepackage{graphicx}
\usepackage{booktabs}
\usepackage{hyperref}
\usepackage{algorithm}
\usepackage{algpseudocode}
\usepackage{enumitem}
\usepackage{titlesec}
\usepackage{paracol} 

\usepackage{tocloft}  

\renewcommand{\numberline}[1]{} 

\newtheorem{proposition}{Proposition}
\newtheorem{definition}{Definition}
\hypersetup{
    colorlinks=true,
    linkcolor=blue,
    citecolor=blue,
    urlcolor=blue
}

\titleformat{\section}
  {\large\bfseries}{\thesection}{0.8em}{}
\titleformat{\subsection}
  {\normalsize\bfseries}{\thesubsection}{0.6em}{}

\makeatletter
\renewcommand{\maketitle}{
  \begin{center}
    \rule{\textwidth}{0.4pt} \\[0.6em]
    {\LARGE\bfseries \@title \par}
    \vspace{0.5em}
    {\large \@author \par}
    \vspace{0.4em}
    {\small \@date \par}
    \rule{\textwidth}{0.4pt} \\[1.2em]
  \end{center}
}
\makeatother

\title{Formal Algorithms for Model Efficiency}

\author{
\vspace{0.8cm} 
\begin{tabular}{c c c}
Naman Tyagi$^*$ & Srishti Das$^*$ & Kunal$^*$ \\
\texttt{namantyagi466@gmail.com} & \texttt{srishtidas409@gmail.com} & \texttt{kunal032005@gmail.com} \\[1em]
\multicolumn{3}{c}{Vatsal Gupta$^*$} \\ 
\multicolumn{3}{c}{\texttt{gupta.vatsal2004@gmail.com}} \\
\end{tabular}
\\[1em]
$^*$Department of Computer Science and Engineering, Amity School of Engineering Technology, Amity University, Noida, 201313
}

\date{}

\begin{document}
\maketitle

\begin{abstract}
We introduce the Knob--Meter--Rule (KMR) framework, a unified formalism for representing and reasoning about model efficiency techniques in deep learning. By abstracting diverse methods—including pruning, quantization, knowledge distillation, and parameter-efficient architectures—into a consistent set of controllable knobs, deterministic rules, and measurable meters, KMR provides a mathematically precise and modular perspective on efficiency optimization. The framework enables systematic composition of multiple techniques, flexible policy-driven application, and iterative budgeted optimization through the \textsc{Budgeted-KMR} algorithm. We demonstrate how well-known efficiency methods can be instantiated as KMR triples and present concise algorithmic templates for each. The framework highlights underlying relationships between methods, facilitates hybrid pipelines, and lays the foundation for future research in automated policy learning, dynamic adaptation, and theoretical analysis of cost-quality trade-offs. Overall, KMR offers both a conceptual and practical tool for unifying and advancing model efficiency research.
\end{abstract}

\begin{paracol}{2}
\switchcolumn[0] 

{\footnotesize        
\setlength{\parskip}{0pt}    
\setlength{\itemsep}{0pt}    
\renewcommand{\contentsname}{Outline} 
\tableofcontents
}

\switchcolumn 
\section{Introduction}

Deep neural networks have transformed fields ranging from computer vision to natural language processing, achieving state-of-the-art performance on many tasks. However, their remarkable capabilities often come with substantial computational, memory, and energy costs. Large-scale models can be impractical to deploy in resource-constrained environments, such as mobile devices, edge computing, or real-time systems. Consequently, developing methods to reduce model size, computation, or energy consumption—collectively referred to as \emph{model efficiency techniques}—has become a central concern in modern machine learning.
\end{paracol}

\newpage
A wide variety of approaches have been proposed to improve efficiency, including \emph{pruning}, \emph{quantization}, \emph{low-rank approximation}, \emph{knowledge distillation}, and \emph{parameter-efficient fine-tuning}. Each method offers different trade-offs between model size, accuracy, and computation, and they are often presented in disparate forms: as architectural modifications, optimization procedures, or training recipes. This heterogeneity makes it difficult to compare methods systematically or reason about their underlying principles.

In this work, we present a \emph{formal, algorithmic framework} for model efficiency. We show that seemingly distinct techniques can be described using a minimal set of constructs: (i) a \emph{model} with parameters, (ii) a \emph{constraint or budget} defining efficiency, (iii) a \emph{deterministic transformation} applied to the model, and (iv) an optional \emph{refinement step} such as retraining or fine-tuning. Using this representation, we provide concise algorithmic descriptions of the major efficiency methods, highlighting their common structure while preserving their distinct characteristics. 

Our framework has three key benefits. First, it unifies a diverse set of techniques under a single, coherent perspective. Second, it provides a clear and formal language for describing and comparing methods without relying on heavy theoretical machinery or proofs. Third, it establishes a foundation for modular composition and future extensions, enabling systematic exploration of hybrid efficiency strategies.

In Section 1, we introduce the motivation for studying model efficiency and present the Knob–Meter–Rule (KMR) framework as a unified formalism. Section 2 reviews related work on pruning, quantization, knowledge distillation, and parameter-efficient architectures, highlighting similarities and differences. Section 3 establishes the necessary preliminaries and formal definitions for models, parameters, knobs, meters, and rules. In Section 4, we present the unified KMR framework, formalize the efficiency problem, and introduce the general Budgeted-KMR algorithm. Section 5 demonstrates how specific efficiency methods can be instantiated as knob–rule–meter triples and provides concise algorithmic templates for each. Section 6 shows how multiple KMR instantiations can be composed into hybrid pipelines and discusses iterative application under policies. Finally, Sections 7 and 8 provide a discussion of the framework’s benefits, limitations, and potential extensions, followed by a conclusion summarizing the main contributions and avenues for future work.

\stepcounter{section} 

\section{Related Work}

Research on model efficiency spans several decades, evolving alongside the growth of deep learning models. While early work focused on reducing computational redundancy in classical machine learning models, the emergence of large-scale neural networks has driven a surge in methods for compressing and accelerating deep architectures.

\textbf{Pruning.} Neural network pruning originated with early weight magnitude heuristics~\cite{lecun1990optimal, hassibi1992second}, later refined into structured pruning techniques that remove entire channels, filters, or layers~\cite{li2017pruningfiltersefficientconvnets, he2017channelpruningacceleratingdeep, molchanov2019importanceestimationneuralnetwork}. Pruning has been shown to reduce parameter count and inference cost with minimal accuracy loss, but methods differ widely in their criteria, granularity, and retraining strategies.

\textbf{Quantization.} Quantization methods map model weights and activations to lower-precision representations, reducing memory and compute cost~\cite{gong2014compressingdeepconvolutionalnetworks, pmlr-v37-gupta15, Jacob_2018_CVPR}. Advances in mixed-precision training~\cite{micikevicius2018mixedprecisiontraining} and post-training quantization~\cite{Nagel2020UpOD} have enabled deployment of efficient models without extensive retraining.

\textbf{Knowledge Distillation.} Distillation transfers knowledge from a high-capacity teacher model to a smaller student~\cite{hinton2015distillingknowledgeneuralnetwork}, with extensions including intermediate feature matching~\cite{romero2015fitnetshintsdeepnets}, self-distillation~\cite{9008829}, and multi-teacher frameworks~\cite{You2017LearningFM}. While effective, distillation often operates orthogonally to other efficiency methods, enabling hybrid approaches.

\textbf{Architecture Modification.} Architectural efficiency can be achieved via manual or automated design, including depth/width scaling~\cite{tan2019efficientnet}, low-rank factorization~\cite{Denil2013PredictingPI, Denton2014ExploitingLS}, and neural architecture search with efficiency constraints~\cite{Cai2019OnceFA, Tan2018MnasNetPN}.

\textbf{Unified Frameworks and Taxonomies.} Surveys have categorized efficiency techniques by their target (weights, activations, architecture) or by compression strategy~\cite{Cheng2017ASO, Choudhary2020ACS}. However, most existing frameworks remain descriptive, lacking a mathematical formalism for reasoning about method composition or interaction. A few works attempt unified representations, such as computational graphs for efficiency pipelines~\cite{iana-etal-2023-newsreclib} or meta-learning approaches for joint optimization~\cite{Elbayad2019DepthAdaptiveT}, but these do not generalize to all operator families.

\textbf{Our Contribution.} In contrast, we introduce an operator-based formalism that treats each efficiency method as a well-defined transformation from one model state to another, with explicit domains, codomains, and measurable effects. This perspective enables rigorous analysis, composability, and systematic exploration of efficiency pipelines, bridging the gap between descriptive taxonomies and algorithmic theory.

\section{Preliminaries and Definitions}

We begin by formalizing the key components of model efficiency techniques in a minimal, self-contained framework.

\subsection{Model and Parameters}

Let $M$ denote a neural network model with parameters $\theta \in \mathbb{R}^n$. The model is trained on a dataset $D \subseteq \mathcal{X} \times \mathcal{Y}$, where $\mathcal{X}$ is the input space and $\mathcal{Y}$ is the label space, using a loss function $\mathcal{L}(\theta; D)$:
\[
\mathcal{L}(\theta; D) = \frac{1}{|D|} \sum_{(x,y) \in D} \ell(f_\theta(x), y),
\]
where $f_\theta: \mathcal{X} \to \mathcal{Y}$ represents the function computed by the model, and $\ell: \mathcal{Y} \times \mathcal{Y} \to \mathbb{R}_+$ is an appropriate task-specific loss.

\subsection{Knobs, Meters, and Rules (KMR)}

We formalize efficiency transformations using three constructs:

\begin{definition}[Knob]
A \emph{knob} is an element $k \in \mathcal{K}$ from a finite set of controllable parameters.  
Each knob $k$ has an associated domain of permissible values $\mathrm{Dom}(k) \subseteq \mathbb{R}$ (or more generally, a discrete/continuous subset of $\mathbb{R}$).  
Formally, a knob is represented as the pair $(k, \mathrm{Dom}(k))$.  
Examples include prune fraction $k_p \in [0,1]$, bitwidth $k_b \in \{2,4,8,16,32\}$, rank $k_r \in \mathbb{N}$, or student size $k_s \in \mathbb{N}_+$.
\end{definition}

\begin{definition}[Meter]
A \emph{meter} is a scalar function that evaluates a property of a model. We define two meters:
\begin{itemize}
    \item $C: \mathcal{M} \to \mathbb{R}_+$, a \emph{cost meter}, measuring computation, memory, or latency.
    \item $Q: \mathcal{M} \to \mathbb{R}$, a \emph{quality meter}, representing model accuracy, validation score, or a proxy for performance.
\end{itemize}
Here, $\mathcal{M}$ is the set of all possible models derived from $M$ through rules and training.
\end{definition}

\begin{definition}[Rule]
Let $\mathcal{T}$ denote the set of transformation rules.  
A \emph{rule} $T \in \mathcal{T}$ is a deterministic transformation that applies a knob to a model:
\[
T: (M, k, v) \mapsto M',
\]
where $k \in \mathcal{K}$ is a knob and $v \in \mathrm{Dom}(k)$ is its value. Examples include pruning the smallest-magnitude weights, quantizing weights to a codebook, low-rank factorization, or distilling a student model.
\end{definition}

\subsection{Efficiency Problem}

Given a base model $M_0 \in \mathcal{M}$ and a cost budget $B \in \mathbb{R}_+$, the \emph{efficiency problem} is to find a model $M' \in \mathcal{M}$ satisfying:
\[
C(M') \le B \quad \text{and} \quad Q(M') \text{ is maximized}.
\]

A generic efficiency pipeline iteratively selects knobs and applies rules to reduce cost while maintaining quality. Optional retraining or fine-tuning can be included to recover performance after each transformation.

\subsection{Policy}

A \emph{policy} is a function $\pi: \mathcal{M} \times \mathbb{R}_+ \times \mathbb{R} \to \mathcal{K} \times \bigcup_{k \in \mathcal{K}} \mathrm{Dom}(k)$ that determines which knob and value to apply next, based on the current model and meters:
\[
\pi(M, C(M), Q(M)) \mapsto (k, v).
\]

Common policies include:
\begin{itemize}
    \item Greedy selection: pick the knob/value pair with highest quality gain per unit cost reduction.
    \item Scheduled order: apply knobs in a fixed sequence (e.g., prune $\rightarrow$ quantize $\rightarrow$ distill).
    \item Dual controller: optimize a scalar objective $U(M) = Q(M) - \lambda C(M)$ with dynamic $\lambda$.
\end{itemize}

\section{Unified Framework for Efficiency Transformations}

With the preliminaries in place, we now unify efficiency methods under a single formal framework. The key idea is that applying efficiency techniques can be described as a sequence of knob–rule applications, guided by meters.

\subsection{Transformation Sequence}

\begin{definition}[Transformation Sequence]
Given an initial model $M_0$, a \emph{transformation sequence} is a finite sequence of rules
\[
\mathcal{S} = (T_1, T_2, \dots, T_m),
\]
where each $T_i : (M_{i-1}, k_i, v_i) \mapsto M_i$ is a rule applied using knob $k_i \in \mathcal{K}$ and value $v_i \in \mathrm{Dom}(k_i)$. The resulting sequence of models is
\[
M_0 \xrightarrow{T_1} M_1 \xrightarrow{T_2} M_2 \xrightarrow{} \cdots \xrightarrow{T_m} M_m.
\]
\end{definition}

\subsection{Feasible Models}

\begin{definition}[Feasibility]
A model $M$ is \emph{feasible} under budget $B$ if
\[
C(M) \leq B.
\]
The set of all feasible models obtained from $M_0$ via any transformation sequence is
\[
\mathcal{F}(M_0, B) \;=\; \{\, M' \mid M' \text{ reachable from } M_0 \text{ via } \mathcal{S},\; C(M') \leq B \,\}.
\]
\end{definition}

\subsection{Efficiency Objective}

The efficiency objective is expressed as an optimization over the feasible set:
\[
M^\star \;=\; \operatorname*{arg\,max}_{M' \in \mathcal{F}(M_0, B)} \; Q(M').
\]

\subsection{Generic KMR Algorithm}
\begin{algorithm}[H]
\caption{Budgeted-KMR}
\label{alg:budgeted-kmr}
\begin{algorithmic}[1]
\Require Initial model $M_0$; budget $B$; meters $C,Q$; knob set $\mathcal{K}$; rule set $\mathcal{T}$; policy $\pi$; dataset $D$; max iterations $N$
\State $M \gets M_0$, $\text{iter} \gets 0$
\While{$C(M) > B$ \textbf{and} $\text{iter} < N$}
    \State $(k, v) \gets \pi(M, C(M), Q(M))$ 
    \State $T \gets \text{GetRule}(k, \mathcal{T})$ \Comment{deterministic rule selection}
    \State $M' \gets T(M, k, v)$ 
    \If{$C(M') \geq C(M)$} \Comment{ensure progress}
        \State \textbf{break} \Comment{no cost reduction possible}
    \EndIf
    \State $M \gets M'$
    \State $M \gets \textsc{FineTune}(M, D)$ \Comment{optional}
    \State $\text{iter} \gets \text{iter} + 1$
\EndWhile
\If{$C(M) > B$}
    \State \textbf{return} \textsc{Failure} \Comment{budget unachievable}
\EndIf
\State \Return $M$
\end{algorithmic}
\end{algorithm}

\begin{proposition}[Monotonicity and Termination]
Let $M_t$ denote the model at the start of the $t$-th iteration of Algorithm~\ref{alg:budgeted-kmr}, with $M_0$ the initial model. Then:
\begin{enumerate}
  \item (Monotonicity) The sequence $\{C(M_t)\}_{t\ge 0}$ is non-increasing. Moreover, for every iteration that completes without breaking, $C(M_{t+1}) < C(M_t)$ (strict decrease).
  \item (Termination) The algorithm terminates in finite time. Specifically, it halts no later than:
  \[
    T_{\max} \le N
  \]
  iterations; it may terminate earlier if $C(M_t)\le B$ (budget satisfied) or if a chosen rule does not reduce cost (line with the \texttt{break}).
\end{enumerate}
\end{proposition}

\begin{proof}[Proof sketch]
Monotonicity follows directly from the acceptance test in the loop: a proposed transformation $M' \leftarrow T(M,k,v)$ is only accepted when $C(M') < C(M)$; otherwise the loop breaks and no update occurs. Therefore, across accepted iterations the cost strictly decreases; across the whole sequence (including the terminal state), the recorded cost is non-increasing.

Termination holds because the loop increments the iteration counter on each accepted iteration and the loop condition enforces $\mathrm{iter} < N$. Thus the algorithm cannot run more than $N$ accepted iterations. It may also stop earlier if the budget $B$ is reached or if the first non-progressing rule is encountered (the \texttt{break} branch).
\end{proof}

\paragraph{Complexity remark.}
Let $\tau_{\text{rule}}$ denote the worst-case time to evaluate and apply a rule (including any cheap simulations the policy uses), let $\tau_{\text{ft}}$ denote the (optional) time for a single fine-tuning call, and let $\tau_{\text{over}}$ denote overhead per iteration (policy, meter evaluation, bookkeeping). Then the wall-clock work of the algorithm is bounded symbolically by
\[
\mathcal{O}\big( N \cdot (\tau_{\text{rule}} + \tau_{\text{ft}} + \tau_{\text{over}}) \big).
\]
This is a high-level, symbolic bound: actual runtime depends on implementation choices (how expensive rule evaluation or fine-tuning is) and on whether the algorithm terminates before reaching $N$.

\section{Instantiations of the KMR Framework}
\label{sec:instantiations}

The KMR calculus introduced above is intentionally abstract, but its power lies in the fact that a wide variety of model efficiency techniques can be expressed as special cases. In this section, we show how well-known families of methods can be \emph{instantiated} as knob--meter--rule (KMR) tuples.

Formally, an instantiation specifies:
\begin{itemize}
    \item A \textbf{knob set} $\mathcal{K}$, where each $k \in \mathcal{K}$ represents a tunable degree of freedom (e.g., layer width, bit precision, rank).
    \item A \textbf{rule set} $\mathcal{T}$, where each $T_k \in \mathcal{T}$ is a deterministic transformation that modifies a model $M$ under an assignment $v$ of knob $k$.
    \item A collection of \textbf{meters} $\{C_i, Q_j\}$, capturing resource costs (e.g., parameter count, latency, memory footprint) and quality measures (e.g., accuracy, loss, robustness).
\end{itemize}

Each efficiency method is therefore represented by a tuple
\[
\mathfrak{I} = \big(\mathcal{K}, \mathcal{T}, \{C_i\}, \{Q_j\}\big).
\]

Crucially, the \textsc{Budgeted-KMR} algorithm (\autoref{alg:budgeted-kmr}) can operate agnostically to the specific instantiation: it iteratively applies rules until a cost budget is met while maintaining quality above a threshold, regardless of whether those rules implement pruning, quantization, distillation, low-rank adaptation, or other transformations.

In the following subsections, we provide canonical instantiations for major categories of efficiency techniques:
\begin{enumerate}
    \item \textbf{Pruning}, where knobs control sparsity ratios and rules remove parameters.
    \item \textbf{Quantization}, where knobs control precision levels and rules map weights or activations to lower bit-width representations.
    \item \textbf{Distillation}, where knobs control student model capacity and rules transfer knowledge from a teacher model.
    \item \textbf{Architectural and Parameter-Efficient Methods}, where knobs govern design factors (e.g., depth, width, rank, adapter size), and rules reparameterize or augment model structure (e.g., LoRA, adapters, tensor decomposition).
\end{enumerate}

This mapping illustrates that, despite their apparent diversity, these approaches share a common formal structure. The KMR framework thus provides both a unifying lens for analysis and a modular scaffold for composing hybrid efficiency strategies.

\subsection{Pruning as a KMR Instantiation}
\label{sec:pruning}

Pruning methods aim to reduce the parameter count of a model by removing weights, neurons, or entire structural components while maintaining predictive performance. Within the KMR framework, pruning can be naturally represented as follows.

\paragraph{Knobs.}  
The knob corresponds to a sparsity level parameter:
\[
k_{\text{prune}} \in [0,1],
\]
where $k_{\text{prune}} = 0$ denotes a fully dense model and $k_{\text{prune}} = 1$ denotes complete removal of parameters. More generally, knobs may be defined per layer, e.g., $k^{(\ell)}_{\text{prune}}$ for layer $\ell$, allowing fine-grained control of sparsity patterns.

\paragraph{Rule.}  
The pruning rule is a transformation
\[
T_{\text{prune}}(M, k_{\text{prune}}) \mapsto M',
\]
which removes a fraction $k_{\text{prune}}$ of parameters according to a criterion (e.g., magnitude-based, gradient-based, or structured pruning). The rule is deterministic given the model $M$ and knob value.

\paragraph{Meters.}  
The meters evaluate:
\[
C(M) = \text{ParamCount}(M), \qquad Q(M) = \text{Accuracy}(M;\mathcal{D}_{\text{val}}).
\]
Additional cost meters such as FLOPs or latency can also be tracked to reflect hardware-aware pruning.

\paragraph{Instantiation.}  
Thus, pruning corresponds to the instantiation
\[
\mathfrak{I}_{\text{prune}} = \big( \{k_{\text{prune}}\}, \{T_{\text{prune}}\}, \{C\}, \{Q\} \big),
\]
where increasing the knob value monotonically decreases parameter count while potentially degrading quality.

This formalization subsumes common pruning variants: unstructured pruning (individual weights), structured pruning (channels, filters), and dynamic pruning (masking at runtime). In all cases, the pruning method can be expressed as a knob–rule–meter triple within the KMR calculus.

\begin{algorithm}[H]
\caption{\textsc{Generic Pruning Procedure}}
\label{alg:pruning}
\begin{algorithmic}[1]
\Require Model $M$, knob value $k_{\text{prune}}$, pruning criterion $\pi$, training data $\mathcal{D}$
\Ensure Pruned model $M'$
\State Compute importance scores $s = \pi(M, \mathcal{D})$ \Comment{e.g., weight magnitude, gradients}
\State Identify fraction $k_{\text{prune}}$ of parameters with lowest scores
\State Remove (or mask) selected parameters to obtain $M' = T_{\text{prune}}(M, k_{\text{prune}})$
\State Optionally fine-tune $M'$ on $\mathcal{D}$
\State \Return $M'$
\end{algorithmic}
\end{algorithm}

\begin{algorithm}[H]
\caption{\textsc{Unstructured Pruning}}
\label{alg:unstructured-pruning}
\begin{algorithmic}[1]
\Require Model $M$, knob $k_{\text{prune}}$ (sparsity ratio), pruning criterion $\pi$, dataset $\mathcal{D}$
\Ensure Pruned model $M'$
\State Compute importance scores $s = \pi(M, \mathcal{D})$ \Comment{e.g., weight magnitudes}
\State Select the bottom $k_{\text{prune}} \times |W|$ weights with lowest scores
\State Zero out or mask selected weights: $M' = T_{\text{unstruct}}(M, k_{\text{prune}})$
\State Optionally fine-tune $M'$ on $\mathcal{D}$
\State \Return $M'$
\end{algorithmic}
\end{algorithm}

\begin{algorithm}[H]
\caption{\textsc{Structured Pruning}}
\label{alg:structured-pruning}
\begin{algorithmic}[1]
\Require Model $M$, knob $k_{\text{prune}}$ (fraction of structures), pruning criterion $\pi$, dataset $\mathcal{D}$
\Ensure Pruned model $M'$
\State Compute importance scores for structures $s = \pi(M, \mathcal{D})$ \Comment{e.g., channels, filters, heads}
\State Select the bottom $k_{\text{prune}} \times |\mathcal{S}|$ structures with lowest scores
\State Remove selected structures: $M' = T_{\text{struct}}(M, k_{\text{prune}})$
\State Adjust downstream dimensions (e.g., layer shapes, tensor sizes)
\State Optionally fine-tune $M'$ on $\mathcal{D}$
\State \Return $M'$
\end{algorithmic}
\end{algorithm}

\subsection{Quantization}
\label{sec:quantization}

Quantization reduces model size and inference cost by constraining weights and activations to lower-precision representations. Within the KMR calculus, quantization can be expressed as follows:

\begin{itemize}
    \item \textbf{Knob:} $k_{\text{quant}} \in \{1,2,4,8,16,32\}$, denoting the bitwidth allocated to weights or activations.
    \item \textbf{Rule:} $T_{\text{quant}}(M, k_{\text{quant}})$ maps real-valued parameters $w \in \mathbb{R}$ to discrete levels in a finite set of representable values:
    \[
        T_{\text{quant}}(w) = \Delta \cdot \text{round}\!\left(\frac{w}{\Delta}\right), \quad \Delta = \frac{\max(w) - \min(w)}{2^{k_{\text{quant}}}-1}.
    \]
    \item \textbf{Meters:} 
        \begin{itemize}
            \item $C(M)$ tracks storage/memory footprint and latency under quantized arithmetic.
            \item $Q(M)$ evaluates the model’s accuracy or loss post-quantization.
        \end{itemize}
\end{itemize}

Thus, quantization trades numerical precision for efficiency, with the knob $k_{\text{quant}}$ directly controlling the budget--quality tradeoff. Crucially, the \textsc{Budgeted-KMR} algorithm applies identically: it iteratively reduces precision until the desired budget is satisfied.

We distinguish two canonical algorithmic instantiations:
\begin{enumerate}
    \item \textbf{Post-Training Quantization (PTQ)}: quantization is applied directly to a pretrained model with minimal or no retraining.
    \item \textbf{Quantization-Aware Training (QAT)}: quantization effects are simulated during training, allowing the model to adapt to discrete representations.
\end{enumerate}

\begin{algorithm}[t]
\caption{Post-Training Quantization (PTQ)}
\label{alg:ptq}
\begin{algorithmic}[1]
\Require Pretrained model $M$; quantization knob $k_{\text{quant}}$; dataset $D$ (for calibration); budget $B$
\State $M_q \gets T_{\text{quant}}(M, k_{\text{quant}})$ \Comment{apply quantization rule}
\State Evaluate $Q(M_q)$ on calibration dataset $D$
\If{$C(M_q) \leq B$}
    \State \Return $M_q$
\Else
    \State \Return \textsc{Failure} \Comment{budget not met}
\EndIf
\end{algorithmic}
\end{algorithm}

\begin{algorithm}[t]
\caption{Quantization-Aware Training (QAT)}
\label{alg:qat}
\begin{algorithmic}[1]
\Require Initial model $M_0$; quantization knob $k_{\text{quant}}$; training dataset $D$; epochs $E$
\State $M \gets M_0$
\For{$e = 1$ to $E$}
    \State Simulate quantization: $\tilde{M} \gets T_{\text{quant}}(M, k_{\text{quant}})$
    \State Compute loss $\mathcal{L}(\tilde{M}, D)$
    \State Backpropagate gradients w.r.t. $M$ \Comment{STE for non-differentiability}
    \State Update parameters of $M$
\EndFor
\State $M_q \gets T_{\text{quant}}(M, k_{\text{quant}})$ \Comment{final quantized model}
\State \Return $M_q$
\end{algorithmic}
\end{algorithm}

\subsection{Knowledge Distillation}
\label{sec:distillation}

Knowledge distillation reduces model size and improves efficiency by training a smaller \emph{student} model to mimic a larger \emph{teacher} model. In the KMR framework, distillation is represented as follows:

\paragraph{Knobs.}  
Knobs control the student model’s capacity and distillation parameters:
\[
k_{\text{student}} \in \mathcal{K}_{\text{student}}, \quad
k_{\text{temp}} \in \mathcal{K}_{\text{temp}}, \quad
k_{\text{loss}} \in \mathcal{K}_{\text{loss}},
\]
where $k_{\text{student}}$ may specify layer width or depth, $k_{\text{temp}}$ the softmax temperature, and $k_{\text{loss}}$ the weighting between teacher and ground-truth loss.

\paragraph{Rule.}  
The distillation rule is
\[
T_{\text{distill}}(M_{\text{teacher}}, M_{\text{student}}, k) \mapsto M'_{\text{student}},
\]
which trains the student model to minimize a combined loss:
\[
\mathcal{L}_{\text{KD}} = (1 - k_{\text{loss}}) \cdot \ell(M_{\text{student}}, y) + k_{\text{loss}} \cdot \text{KL}\Big(\sigma(M_{\text{student}}/k_{\text{temp}}), \sigma(M_{\text{teacher}}/k_{\text{temp}})\Big),
\]
where $\sigma$ denotes the softmax function.

\paragraph{Meters.}  
\[
C(M) = \text{Params}(M_{\text{student}}), \quad Q(M) = \text{Accuracy}(M_{\text{student}};\mathcal{D}_{\text{val}}).
\]

\paragraph{Instantiation.}  
Thus, the distillation method is represented by:
\[
\mathfrak{I}_{\text{distill}} = \big( \{k_{\text{student}}, k_{\text{temp}}, k_{\text{loss}}\}, \{T_{\text{distill}}\}, \{C\}, \{Q\} \big),
\]
which captures a wide range of distillation variants, from small student networks to intermediate feature distillation.

\subsubsection*{Algorithm: Knowledge Distillation}
\begin{algorithm}[H]
\caption{Generic Knowledge Distillation}
\label{alg:distillation}
\begin{algorithmic}[1]
\Require Teacher model $M_T$; student model $M_S$; knob settings $k$; training dataset $D$; epochs $E$
\State Initialize student model $M_S$ according to $k_{\text{student}}$
\For{$e = 1$ to $E$}
    \State Compute student predictions $p_S = M_S(x)$ for $x \in D$
    \State Compute teacher predictions $p_T = M_T(x)$
    \State Compute distillation loss $\mathcal{L}_{\text{KD}}$ using $k_{\text{temp}}, k_{\text{loss}}$
    \State Backpropagate and update $M_S$
\EndFor
\State \Return Trained student model $M_S$
\end{algorithmic}
\end{algorithm}

\subsection{Architectural and Parameter-Efficient Methods}
\label{sec:architectural}

These methods improve efficiency by modifying the model’s architecture or introducing low-rank or adapter modules, reducing parameter count or computational cost while maintaining performance.

\paragraph{Knobs.}  
Knobs control architectural and low-rank factors:
\[
k_{\text{depth}}, k_{\text{width}}, k_{\text{rank}}, k_{\text{adapter}} \in \mathcal{K}_{\text{arch}},
\]
where $k_{\text{depth}}$ and $k_{\text{width}}$ govern layer depth and width, $k_{\text{rank}}$ specifies rank in low-rank factorization (e.g., LoRA), and $k_{\text{adapter}}$ specifies adapter size or inserted layers.

\paragraph{Rule.}  
The transformation is
\[
T_{\text{arch}}(M, k) \mapsto M',
\]
where $T_{\text{arch}}$ reconfigures the model according to the knob assignments:
\begin{itemize}
    \item Adjusting layer width or depth,
    \item Injecting adapters or LoRA modules with rank $k_{\text{rank}}$,
    \item Applying low-rank factorization to weight matrices.
\end{itemize}

\paragraph{Meters.}  
\[
C(M) = \text{Params}(M) \text{ or } \text{FLOPs}(M), \qquad
Q(M) = \text{Accuracy}(M; \mathcal{D}_{\text{val}}).
\]

\paragraph{Instantiation.}  
The KMR tuple is
\[
\mathfrak{I}_{\text{arch}} = \big( \{k_{\text{depth}}, k_{\text{width}}, k_{\text{rank}}, k_{\text{adapter}}\}, \{T_{\text{arch}}\}, \{C\}, \{Q\} \big),
\]
capturing a variety of methods including LoRA, adapters, and low-rank reparameterizations.

\subsubsection*{Algorithm: Parameter-Efficient Fine-Tuning}
\begin{algorithm}[H]
\caption{Generic Low-Rank / Adapter Fine-Tuning}
\label{alg:arch-efficient}
\begin{algorithmic}[1]
\Require Base model $M_0$; knob settings $k$; training dataset $D$; epochs $E$
\State Configure model $M$ according to knobs ($k_{\text{depth}}, k_{\text{width}}, k_{\text{rank}}, k_{\text{adapter}}$)
\For{$e = 1$ to $E$}
    \State Forward pass through $M$
    \State Compute task loss $\mathcal{L}(M, D)$
    \State Backpropagate gradients and update trainable components (e.g., adapters, LoRA matrices)
\EndFor
\State \Return Fine-tuned, parameter-efficient model $M$
\end{algorithmic}
\end{algorithm}

\subsection{Other Transformations}
\label{sec:other-transformations}

Beyond the canonical categories, many additional efficiency techniques can be captured by the KMR framework. These include tensor decompositions, weight sharing, structured sparsity, and hybrid methods that combine multiple knobs.

\paragraph{Knobs.}  
Knobs for other transformations may include:
\[
k_{\text{tensor}}, k_{\text{share}}, k_{\text{hybrid}} \in \mathcal{K}_{\text{other}},
\]
where $k_{\text{tensor}}$ specifies ranks for tensor decompositions (e.g., CP, Tucker, or tensor-train), $k_{\text{share}}$ controls the degree of parameter sharing across layers or modules, and $k_{\text{hybrid}}$ represents combination strategies or scaling factors for multi-technique methods.

\paragraph{Rule.}  
Each transformation is implemented as a deterministic rule
\[
T_{\text{other}}(M, k) \mapsto M',
\]
for example:
\begin{itemize}
    \item Applying low-rank tensor decomposition to weight tensors,
    \item Sharing parameters across layers or blocks according to $k_{\text{share}}$,
    \item Composing multiple efficiency techniques using a hybrid schedule.
\end{itemize}

\paragraph{Meters.}  
\[
C(M) = \text{FLOPs}(M), \text{Memory}(M), \text{ParamCount}(M), \qquad
Q(M) = \text{Accuracy}(M;\mathcal{D}_{\text{val}}) \text{ or other task metrics}.
\]

\paragraph{Instantiation.}  
The KMR tuple is
\[
\mathfrak{I}_{\text{other}} = \big( \{k_{\text{tensor}}, k_{\text{share}}, k_{\text{hybrid}}\}, \{T_{\text{other}}\}, \{C\}, \{Q\} \big),
\]
demonstrating that even non-standard or composite methods can be represented as knob–rule–meter triples.

\subsubsection*{Algorithm: General Other Transformation}
\begin{algorithm}[H]
\caption{Generic Tensor / Sharing / Hybrid Transformation}
\label{alg:other-transformations}
\begin{algorithmic}[1]
\Require Model $M_0$; knob settings $k$; dataset $D$; max iterations $N$
\State $M \gets M_0$, $\text{iter} \gets 0$
\While{$\text{iter} < N$}
    \State Select transformation $T \in \{T_{\text{tensor}}, T_{\text{share}}, T_{\text{hybrid}}\}$
    \State Apply rule: $M' \gets T(M, k)$
    \State Optionally fine-tune $M'$ on $D$
    \State Update meters $C(M'), Q(M')$
    \State $M \gets M'$, $\text{iter} \gets \text{iter} + 1$
\EndWhile
\State \Return Transformed model $M$
\end{algorithmic}
\end{algorithm}

\section{Composing KMR Methods}
\label{sec:composition}

One of the main advantages of the KMR framework is its modularity: multiple efficiency techniques can be composed into a single pipeline while maintaining a unified representation of knobs, rules, and meters. Formally, let
\[
\mathfrak{I}_1, \mathfrak{I}_2, \dots, \mathfrak{I}_n
\]
be KMR instantiations corresponding to different methods (e.g., pruning, quantization, distillation, adapters).

\paragraph{Combined Knobs and Rules.}  
The combined system has:
\[
\mathcal{K}_{\text{combined}} = \bigcup_{i=1}^{n} \mathcal{K}_i, \qquad
\mathcal{T}_{\text{combined}} = \bigcup_{i=1}^{n} \mathcal{T}_i,
\]
where each knob $k \in \mathcal{K}_{\text{combined}}$ is associated with its corresponding transformation rule $T_k \in \mathcal{T}_{\text{combined}}$.

\paragraph{Combined Meters.}  
Meters track both resource usage and model quality across all transformations:
\[
C_{\text{combined}}(M) = \text{AggregateCost}(M), \qquad
Q_{\text{combined}}(M) = \text{AggregateQuality}(M),
\]
where $\text{AggregateCost}$ and $\text{AggregateQuality}$ may consider weighted sums of individual meters or task-specific priorities.

\paragraph{Composed Budgeted-KMR Algorithm.}  
The iterative Budgeted-KMR algorithm (\autoref{alg:budgeted-kmr}) can be applied directly to the combined instantiation:
\begin{enumerate}
    \item Select the next knob $k \in \mathcal{K}_{\text{combined}}$ according to a policy $\pi$.
    \item Apply its rule $T_k$ to update the model.
    \item Evaluate meters $(C_{\text{combined}}, Q_{\text{combined}})$.
    \item Continue until the cost budget is satisfied or no further improvements are possible.
\end{enumerate}

\paragraph{Flexibility and Scheduling.}  
- Techniques can be applied sequentially (e.g., prune $\rightarrow$ quantize $\rightarrow$ distill) or interleaved based on policy.  
- Knobs may be adjusted dynamically depending on intermediate meter readings, allowing adaptive allocation of resources.  
- Hybrid strategies can exploit synergies between methods, e.g., pruning followed by low-rank decomposition to further reduce memory footprint.

This compositional view demonstrates that the KMR framework naturally supports \textbf{multi-technique pipelines}, providing a unified, mathematically precise lens for reasoning about complex model efficiency strategies.

\begin{algorithm}[t]
\caption{Composed Budgeted-KMR for Multiple Instantiations}
\label{alg:composed-budgeted-kmr}
\begin{algorithmic}[1]
\Require Initial model $M_0$; budget $B$; combined knob set $\mathcal{K}_{\text{combined}}$; combined rule set $\mathcal{T}_{\text{combined}}$; meters $(C_{\text{combined}}, Q_{\text{combined}})$; policy $\pi$; dataset $D$; max iterations $N$
\State $M \gets M_0$, $\text{iter} \gets 0$
\While{$C_{\text{combined}}(M) > B$ \textbf{and} $\text{iter} < N$}
    \State $(k, v) \gets \pi(M, C_{\text{combined}}(M), Q_{\text{combined}}(M))$ \Comment{select knob/value}
    \State $T_k \gets \text{GetRule}(k, \mathcal{T}_{\text{combined}})$
    \State $M' \gets T_k(M, k, v)$
    \If{$C_{\text{combined}}(M') \ge C_{\text{combined}}(M)$} \Comment{ensure progress}
        \State \textbf{break} \Comment{no further cost reduction possible}
    \EndIf
    \State $M \gets M'$
    \State $M \gets \textsc{FineTune}(M, D)$ \Comment{optional, e.g., KD, QAT, adapters}
    \State $\text{iter} \gets \text{iter} + 1$
\EndWhile
\If{$C_{\text{combined}}(M) > B$}
    \State \textbf{return} \textsc{Failure} \Comment{budget not achievable}
\EndIf
\State \Return $M$
\end{algorithmic}
\end{algorithm}

\section{Discussion}
\label{sec:discussion}

The Knob–Meter–Rule (KMR) framework provides a unified lens through which to view a wide variety of model efficiency techniques. By formalizing knobs, rules, and meters, we are able to represent seemingly disparate methods—such as pruning, quantization, knowledge distillation, and parameter-efficient architectures—within a single, coherent formalism. This abstraction highlights the underlying similarities between methods that are often treated separately in the literature.

A key benefit of this formalism is its modularity. Each method can be expressed as a specific instantiation of KMR, and multiple instantiations can be combined seamlessly. This composability allows the construction of hybrid pipelines, in which, for example, pruning is followed by low-rank decomposition or quantization, all within a single framework. Moreover, the framework naturally accommodates flexible policies for knob selection, ranging from simple greedy heuristics to adaptive, learned controllers, enabling dynamic and context-aware optimization of model efficiency.

Despite these advantages, there are some limitations to consider. The abstraction deliberately omits implementation-specific details such as optimization dynamics, hardware-specific performance characteristics, or convergence behavior. Additionally, the choice of knob granularity can significantly influence the expressiveness and practicality of the framework: knobs that are too coarse may fail to capture useful degrees of freedom, while overly fine knobs may increase computational complexity or decision space size. Finally, the metrics used for quality and cost may be task- or hardware-dependent, potentially complicating comparisons across different settings.

Looking forward, the KMR framework opens several avenues for further research. One promising direction is the development of automated policy learning strategies, leveraging reinforcement learning or meta-learning to select knobs optimally. Another is the design of composite or hybrid metrics that better capture trade-offs between computation, memory, latency, and performance. Additionally, applying KMR in dynamic or streaming contexts could allow models to adapt their efficiency strategies in real time. Finally, while the current work focuses on operational definitions and instantiations, the framework could serve as a foundation for more formal theoretical analyses of cost-quality trade-offs.

In summary, the KMR framework provides both a mathematically precise and practically flexible approach to model efficiency. It unifies a broad spectrum of techniques, facilitates compositional pipelines, and lays the groundwork for future theoretical and applied studies in resource-constrained machine learning.
\section{Conclusion}
\label{sec:conclusion}

In this work, we introduced the Knob–Meter–Rule (KMR) framework, a unified formalism for representing and reasoning about model efficiency techniques. By abstracting diverse methods—including pruning, quantization, knowledge distillation, and parameter-efficient architectures—into a consistent set of knobs, rules, and meters, we provide a mathematically precise and modular lens through which efficiency methods can be analyzed, compared, and combined.

The \textsc{Budgeted-KMR} algorithm demonstrates how these instantiations can be applied iteratively to meet resource budgets while preserving model quality. Importantly, the framework supports composition of multiple techniques, allowing hybrid pipelines that adaptively exploit synergies between methods. This composability, together with the flexibility in knob-selection policies, highlights KMR’s practical applicability across a range of tasks and deployment scenarios.

Beyond formal representation, KMR lays the groundwork for future research. Automated policy learning, dynamic adaptation of knobs in online settings, and the exploration of composite cost-quality metrics are natural extensions that could further enhance the efficiency and effectiveness of deep learning models. In essence, the framework bridges the gap between theoretical abstraction and practical deployment, offering a structured approach to understanding and designing model efficiency strategies.

Overall, the KMR framework not only unifies existing techniques but also provides a foundation for future innovations in resource-constrained machine learning, fostering both theoretical insight and practical impact.

\bibliographystyle{ieeetr}

\bibliography{references}
\end{document}